\documentclass{article}
\usepackage{times}
\usepackage{graphicx} % more modern
\usepackage{subfigure} 
\usepackage{natbib}

\usepackage{algorithm}
\usepackage{algorithmic}
\usepackage{amsmath,amssymb,amsthm,enumerate,comment}
\usepackage[margin=1.25in]{geometry}
\newtheorem{theorem}{Theorem}
\newtheorem{lemma}{Lemma}

\newcommand{\ap}{\alpha_+}
\newcommand{\la}{\langle}
\newcommand{\rag}{\rangle\!_{{\raisebox{-3pt}{\scriptsize \em G}}}}
\newcommand{\pump}{p_{\mu_+}}
\newcommand{\phat}{\hat{p}}
\newcommand{\half}{\tfrac1{2}}
\newcommand{\onebyn}{\tfrac{\ones_n}{n}}
\newcommand{\F}{\mathcal{F}}

\newcommand{\R}{\mathbb{R}}

\newcommand{\tf}{{\tilde{f}}}
\newcommand{\tphi}{{\tilde{\phi}}}

\newcommand{\ones}{\mathbf{1}}
\newcommand{\zeros}{\mathbf{0}}

\title{Margins, Kernels and Non-linear Smoothed Perceptrons}

\author{
Aaditya Ramdas\\
Machine Learning Department\\
Carnegie Mellon University\\
\texttt{aramdas@cs.cmu.edu}
\and
Javier Pe\~na\\
Tepper School of Business\\
Carnegie Mellon University\\
\texttt{jfp@andrew.cmu.edu}
}

\begin{document} 

\maketitle

\begin{abstract} 
We focus on the problem of finding a non-linear classification function that lies in a Reproducing Kernel Hilbert Space (RKHS) both from the primal point of view (finding a perfect separator when one exists) and the dual point of view (giving a certificate of non-existence), with special focus on generalizations of two classical schemes - the Perceptron (primal) and Von-Neumann (dual) algorithms. 

We cast our problem as one of maximizing the regularized normalized hard-margin ($\rho$) in an RKHS and %use the Representer Theorem to
 rephrase it in terms of a Mahalanobis dot-product/semi-norm associated with the kernel's (normalized and signed) Gram matrix. We derive an accelerated smoothed algorithm with a convergence rate of $\tfrac{\sqrt {\log n}}{\rho}$ given $n$ separable points, which is strikingly similar to the classical kernelized Perceptron algorithm whose rate is $\tfrac1{\rho^2}$. When no such classifier exists, we prove a version of Gordan's separation theorem for RKHSs, and give a reinterpretation of negative margins. This allows us to give guarantees for a primal-dual algorithm that halts in $\min\{\tfrac{\sqrt n}{|\rho|}, \tfrac{\sqrt n}{\epsilon}\}$ iterations with a perfect separator in the RKHS if the primal is feasible or a dual $\epsilon$-certificate of near-infeasibility. 
\end{abstract} 

\section{Introduction}

We are interested in the problem of finding a non-linear separator for a given set of $n$ points $x_1,...,x_n \in \R^d$ with labels $y_1,...,y_n \in \{\pm 1\}$. Finding a linear separator can be stated as the problem of finding a unit vector $w \in \R^d$ (if one exists) such that for all $i$
\begin{equation}\label{prim1}
y_i (w^\top x_i) \geq 0 \mbox{\ \ \ \ \  i.e. \ \ \ sign}(w^\top x_i) = y_i.
\end{equation}
This is called the primal problem. In the more interesting non-linear setting, we will be searching for functions $f$ in a Reproducing Kernel Hilbert Space (RKHS) $\F_K$ associated with kernel $K$ (to be defined later) such that for all $i$
\begin{equation}\label{prim2}
y_i f(x_i) \geq 0 .
\end{equation}
We say that problems (\ref{prim1}), (\ref{prim2}) have an unnormalized margin $\rho > 0$, if there exists a unit vector $w$, such that for all $i$, 
\begin{equation*}
y_i (w^\top x_i) \geq \rho \mbox { \ \ or \ \ } y_i f(x_i) \geq \rho .
\end{equation*}
True to the paper's title, margins of non-linear separators in an RKHS will be a central concept, and we will derive interesting \textit{smoothed accelerated} variants of the Perceptron algorithm that have convergence rates (for the aforementioned primal and a dual problem introduced later) that are inversely proportional to the RKHS-margin as opposed to inverse squared margin for the Perceptron. 

The linear setting is well known by the name of linear feasibility problems - we are asking if there exists any vector $w$ which makes an acute angle with all the vectors $y_i x_i$, i.e.
\begin{equation}\label{lf}
(XY)^\top w > \zeros_n,
\end{equation} 
where $Y := diag(y), X := [x_1,...,x_n]$. This can be seen as finding a vector $w$ inside the dual cone of  $cone\{y_ix_i\}$. 
%Eq.\ref{lf} would be infeasible if the cone defined by $A$ was  $\R^d$ (and hence the dual cone is just the origin), which happens when the origin is in the convex combination of $a_1,. . .,a_n$. In other words, one can show infeasibility by showing the feasibility of the dual problem of finding a probability distribution $ p \in \Delta_n $ such that
%\begin{equation}\label{dual}
%Ap = \zeros_d
%\end{equation}
%where $\Delta_n$ denotes the unit $n$-dimensional simplex $\{\|z\|_1 = 1, z_i \geq 0 \}$. 
%For ease of notation, let $B_{d}$ denote the unit $d$-dimensional $\ell_2$ ball $\|z\|_2 \leq 1$. Let $C_d$ denote the unit $d$-dimensional $\ell_\infty$ cube $\|z\|_\infty \leq 1$.

When normalized, as we will see in the next section, the margin is a well-studied notion of conditioning for these problems. It can be thought of as the width of the feasibility cone as in \cite{FV99}, a radius of well-posedness as in \cite{CC01}, and its inverse can be seen as a special case of a condition number defined by \cite{R95} for these systems.

\subsection{Related Work}

In this paper we focus on the famous Perceptron algorithm from \cite{R58} and the less-famous Von-Neumann algorithm from \cite{D92} that we introduce in later sections. As mentioned by \cite{EF00}, in a technical report by the same name, Nesterov pointed out in a note to the authors that the latter is a special case of the now-popular Frank-Wolfe algorithm.

Our work builds on \cite{SP12,SP13} from the field of optimization - we generalize the setting to learning functions in RKHSs, extend the algorithms, simplify proofs, and simultaneously bring new perspectives to it. There is extensive literature around the Perceptron algorithm in the learning community; we restrict ourselves to discussing only a few directly related papers, in order to point out the several differences from existing work. 

We provide a general unified proof in the Appendix which borrows ideas from accelerated smoothing methods developed by  \cite{N05} - while this algorithm and others by \cite{N04}, \cite{SVZ11} can achieve similar rates for the same problem, those algorithms do not possess the simplicity of the Perceptron or Von-Neumann algorithms and our variants, and also don't look at the infeasible setting or primal-dual algorithms.

Accelerated smoothing techniques have also been seen in the learning literature like in \cite{T08} and many others. However, most of these deal with convex-concave problems where both sets involved are the probability simplex (as in game theory, boosting, etc), while we deal with hard margins where one of the sets is a unit $\ell_2$ ball. Hence, their algorithms/results are not extendable to ours trivially. This work is also connected to the idea of $\epsilon$-coresets by \cite{C10}, though we will not explore that angle.

A related algorithm is called the Winnow by \cite{L91} - this works on the $\ell_1$ margin and is a saddle point problem over two simplices. One can ask whether such accelerated smoothed versions exist for the Winnow. The answer is in the affirmative - however such algorithms look completely different from the Winnow, while in our setting the new algorithms retain the simplicity of the Perceptron. 

%We will give several new insights into the kernelized variants of the classical Perceptron algorithm in this paper - these come from a subtle analysis of margins and normalization of feature maps in RKHSs, connecting the Representer Theorem to the regularized hard margin problem, proving a relevant separation theorem for RKHSs, and showing that a primal-dual Perceptron variant has a convergence rate inversely proportional to the margin to return a separating function or give a certificate of inseparability.

\subsection{Paper Outline}

Sec.2 will introduce the Perceptron and Normalized Perceptron algorithm and their convergence guarantees for linear separability, with specific emphasis on the unnormalized and normalized margins. Sec.3 will then introduce RKHSs and the Normalized Kernel Perceptron algorithm, which we interpret as a subgradient algorithm for a regularized normalized hard-margin loss function.

Sec.4 describes the Smoothed Normalized Kernel Perceptron algorithm that works with a smooth approximation to the original loss function, and outlines the argument for its faster convergence rate. Sec.5 discusses the non-separable case and the Von-Neumann algorithm, and we prove a version of Gordan's theorem in RKHSs. 

We finally give an algorithm in Sec.6 which terminates with a separator if one exists, and with a dual certificate of near-infeasibility otherwise, in time inversely proportional to the margin. Sec.7 has a discussion and some open problems.

\section{Linear Feasibility Problems}

\subsection{Perceptron}
The classical perceptron algorithm can be stated in many ways, one is in the following form

\begin{algorithm}[]
   \caption{Perceptron}
\begin{algorithmic}
   \STATE Initialize $w_0 = 0$
   \FOR{$k=0,1,2,3,...$}
   \IF{sign$(w_k^\top x_i) \neq y_i$ for some $i$} 
   \STATE $w_{k+1} := w_k + y_i x_i$
   \ELSE
   \STATE Halt: Return $w_k$ as solution
   \ENDIF
   \ENDFOR
\end{algorithmic}
\end{algorithm}

It comes with the following classic guarantee as proved by \cite{B62} and \cite{N62}: 
 \textit{If there exists a unit vector $u \in \R^d$ such that $YX^\top u \geq \rho > 0$, then a perfect separator will be found in $\frac{ \max_i \|x_i\|^2_{2}}{\rho^2}$ iterations/mistakes.}

The algorithm works when updated with any arbitrary point $(x_i,y_i)$ that is misclassified; it has the same guarantees when $w$ is updated with the point that is misclassified by the largest amount, $\arg\min_i y_i w^\top x_i$. Alternately, one can define the probability distribution over examples
\begin{equation} \label{pofw}
p(w) = \arg\min_{p \in \Delta_n} \langle YX^\top w, p \rangle,
\end{equation}
where $\Delta_n$ is the $n$-dimensional probability simplex.

Intuitively, $p$ picks the examples that have the lowest margin when classified by $w$.
One can also normalize the updates so that we can maintain a probability distribution over examples used for updates from the start, as seen below:

\begin{algorithm}[]
   \caption{Normalized Perceptron}
\begin{algorithmic}
   \STATE Initialize $w_0 = 0, p_0 = 0$
   \FOR{$k=0,1,2,3,...$}
   \IF{$YX^\top w_k > 0$} 
   \STATE Exit, with $w_k$ as solution
   \ELSE
   \STATE $\theta_k := \frac{1}{k+1}$
   \STATE $w_{k+1} := (1-\theta_k) w_k + \theta_k XYp(w_k)$
   %\STATE $p_{k+1} := (1-\theta_k) p_k + \theta_k p(w_k)$
   \ENDIF
   \ENDFOR
\end{algorithmic}
\end{algorithm}

\paragraph{Remark.} Normalized Perceptron has the same guarantees as perceptron - the Perceptron can perform its update \textit{online} on \textit{any} misclassified point, while the Normalized Perceptron performs updates on the \textit{most} misclassified point(s), and yet there does not seem to be any change in performance. However, we will soon see that the ability to see all the examples at once gives us much more power.

\subsection{Normalized Margins}

If we normalize the data points by the $\ell_2$ norm, the resulting mistake bound of the perceptron algorithm is slightly different. Let $X_2$ represent the matrix with columns $x_i/\|x_i\|_2$. Define the unnormalized and normalized margins as
\begin{eqnarray}
\rho &:=& \sup_{\|w\|_2 = 1} \inf_{p \in \Delta_n} \langle YX^\top w, p \rangle,  \nonumber\\
\rho_2 &:=& \sup_{\|w\|_2=1} \inf_{p \in \Delta_n} \langle YX_2^\top w, p \rangle.  \nonumber
\end{eqnarray}
\paragraph{Remark.} Note that we have $\sup_{\|w\|_2=1}$ in the definition, this is equivalent to $\sup_{\|w\|_2 \leq 1}$ iff $\rho_2 > 0$.

Normalized Perceptron has the following guarantee on $X_2$:
\textit{If $\rho_2 > 0$, then it finds a perfect separator in $\tfrac1{\rho_2^2}$ iterations.}
\paragraph{Remark.} Consider the max-margin separator $u^*$ for $X$ (which is also a valid perfect separator for $X_2$). Then 
\begin{eqnarray*}
\tfrac{\rho}{\max_i \|x_i\|_2} &=& \min_i \left( \frac{y_ix_i^\top u^*}{\max_i \|x_i\|_2} \right) \leq \min_i \left( \frac{y_i x_i^\top u^*}{\|x_i\|_2} \right)\\
&\leq& \sup_{\|u\|_2 = 1}\min_i \left( \frac{y_i x_i^\top u}{\|x_i\|_2} \right) = \rho_2.
\end{eqnarray*}
Hence, it is always better to normalize the data as pointed out in \cite{GHW01}. This idea extends to RKHSs, motivating the normalized Gram matrix considered later.

\textbf{Example} Consider a simple example in $\R^2_+$. Assume that $+$ points are located along the line $6x_2 = 8x_1$, and the $-$ points along $8x_2 = 6x_1$, for $1/r \leq \|x\|_2 \leq r$, where $r~>~1$. The max-margin linear separator will be $x_1=x_2$. If all the data were normalized to have unit Euclidean norm, then all the $+$ points would all be at $(0.6,0.8)$ and all the $-$ points at $(0.8,0.6)$, giving us a normalized margin of $\rho_2 \approx 0.14$. Unnormalized, the margin is $\rho \approx 0.14/r$ and $\max_i \|x_i\|_2 = r$. Hence, in terms of bounds, we get a discrepancy of $r^4$, which can be arbitrarily large.

\textbf{Winnow} The question arises as to which norm we should normalize by. 
There is a now classic algorithm in machine learning, called Winnow by \cite{L91} or Multiplicate Weights. It works on a slight transformation of the problem where we only need to search for $u \in \R^d_+$.
It comes with some very well-known guarantees -  \textit{If there exists a $u \in \R_+^d$ such that $YX^\top u \geq \rho > 0$, then feasibility is guaranteed in $\|u\|_1^2 \max_i \|a_i\|^2_\infty\log n / \rho^2$ iterations.}
The appropriate notion of normalized margin here is
$$
\rho_1 := \max_{w \in \Delta_d} \min_{p \in \Delta_n} \langle YX_\infty^\top w,p\rangle,
$$
where $X_\infty$ is a matrix with columns $x_i/\|x_i\|_\infty$. Then, the appropriate iteration bound is $\log n/\rho_1^2$. We will return to this $\ell_1$-margin in the discussion section. In the next section, we will normalize by using the kernel appropriately.

\section{Kernels and RKHSs}

The theory of Reproducing Kernel Hilbert Spaces (RKHSs) has a rich history, and for a detailed introduction, refer to \cite{SS02}. 
Let $K : \R^d \times \R^d \rightarrow \R$ be a symmetric positive definite  kernel, giving rise to a Reproducing Kernel Hilbert Space $\F_K$ with an associated feature mapping at each point $x \in \R^d$ called $\phi_x : \R^d \rightarrow \F_K$ where $\phi_x (.) = K(x,.)$ i.e. $\phi_x (y) = K(x,y)$. 
$\F_K$ has an associated inner product $\langle \phi_u, \phi_v \rangle_K = K(u,v)$. For any $f \in \F_K$, we have $f(x) = \langle f , \phi_x \rangle_K$. 

Define the normalized feature map 
$$
\tphi_x = \frac{\phi_x}{\sqrt{K(x,x)}} \in \F_K \mbox{ \ \ and \ \ } \tphi_X := [\tphi_{x_i}]_1^n.
$$
For any function $f \in \F_K$, we use the following notation  $$Y\tilde{f}(X) := \langle f, Y\tphi_X \rangle_K = [y_i \langle f,\tphi_{x_i} \rangle_K]_1^n = \Big[ \tfrac{y_i f(x_i)}{\sqrt{K(x_i,x_i)}} \Big]_1^n.$$
We analogously define the normalized margin here to be 
\begin{eqnarray}\label{margin}
\rho_K &:=&  \sup_{\|f\|_K = 1}  \inf_{p \in \Delta_n} \left \langle  Y\tf(X), p \right \rangle.
\end{eqnarray}
Consider the following regularized empirical loss function 
\begin{equation}\label{loss}
L(f) = \left \{ \sup_{p \in \Delta_n} \left \langle -Y\tf(X), p \right \rangle \right \} + \tfrac1{2} \|f\|_K^2.
\end{equation}
Denoting $t := \|f\|_K > 0$ and writing $f = t \left( \frac{f}{\|f\|_K} \right)  = t\bar{f}$, let us calculate the minimum value of this function
\begin{eqnarray}
\inf_{f \in \F_K} L(f) &=& \inf_{t > 0} %\left\{ 
\inf_{\|\bar{f}\|_K = 1}  \sup_{p \in \Delta_n}  \langle - \langle t\bar{f}, Y\tphi_X \rangle_K, p \rangle + \tfrac{t^2}{2} 
%\right\} 
\nonumber\\ 
&=& \inf_{t>0} \left\{  -t \rho_K + \tfrac1{2} t^2 \right\}  \nonumber\\
&=& -\tfrac1{2}\rho_K^2 \mbox{ \ \ when $t=\rho_K>0$. } \label{minloss}
\end{eqnarray}
%We note that Eq.~\ref{last} holds only for positive $\rho_K$ since it is achieved by $\rho_K = t > 0$ (i.e. the data is separable by some function in the RKHS).
Since $\max_{p \in \Delta_n} \left \langle -Y\tf(X), p \right \rangle$ is some empirical loss function on the data and $\tfrac1{2} \|f\|_K^2$ is an increasing function of $\|f\|_K$, the Representer Theorem \citep{SHS01} implies that the minimizer of the above function lies in the span of $\phi_{x_i}$s (also the span of the $y_i\tphi_{x_i}$s). Explicitly,
\begin{eqnarray}
\arg\min_{f \in \F_K} L(f) =  \sum_{i=1}^n \alpha_i y_i\tphi_{x_i} = \langle Y\tphi_X, \alpha \rangle. \label{alpha}
\end{eqnarray}
Substituting this back into Eq.(\ref{loss}), we can define
\begin{eqnarray}
L(\alpha) &:=& \left \{ \sup_{p \in \Delta_n} \left \langle - \alpha, p \right \rag \right \} + \tfrac1{2} \|\alpha\|_G^2,\label{lossag}
%\rho_K &=&  \sup_{\|\alpha\|_G \leq 1}  \inf_{p \in \Delta_n}  \langle \alpha, p \rag 
\end{eqnarray}
where $G$ is a normalized signed Gram matrix with $G_{ii}=1$,
$$G_{ji} = G_{ij} := \tfrac{y_i y_j K(x_i,x_j)}{\sqrt{K(x_i,x_i)K(x_j,x_j)}} = \langle y_i\tphi_{x_i}, y_j\tphi_{x_j} \rangle_K, $$
and $\langle p, \alpha \rag := p^\top G\alpha$, $\|\alpha\|_G := \sqrt{\alpha^\top G \alpha}$. One can verify that $G$ is a PSD matrix and the G-norm $\|.\|_G$ is a semi-norm, whose properties are of great importance to us.

\subsection{Some Interesting and Useful Lemmas}

The first lemma justifies our algorithms' exit condition.
\begin{lemma}\label{Ga}
$L(\alpha) < 0$ implies $G\alpha > 0$ and there exists a perfect classifier iff $G\alpha > 0$. 
\end{lemma}\vspace{-0.2in}
\begin{proof}
$L(\alpha) < 0 \Rightarrow \sup_{p \in \Delta_n} \left \langle - G\alpha, p \right \rangle  < 0 \Leftrightarrow G\alpha > 0$.
$G\alpha > 0 \Rightarrow f_{\alpha} := \langle \alpha, Y\tphi_X \rangle$ is perfect since
\begin{eqnarray*}
\frac{y_j f_\alpha(x_j)}{\sqrt{K(x_j,x_j)}}  &=& \sum_{i=1}^n \alpha_i  \frac{y_i y_j K(x_i, x_j)}{\sqrt{K(x_i,x_i)K(x_j,x_j)}}\\
& = & G_j \alpha > 0.
\end{eqnarray*}
If a perfect classifier exists, then $\rho_K>0$ by definition and 
$$L(f^*) ~=~ L(\alpha^*) = -\half \rho_K^2 < 0 \ \ \ \Rightarrow \ \ \  G\alpha > 0,$$
where $f^*, \alpha^*$ are the optimizers of $L(f), L(\alpha)$.
\end{proof}

\vspace{0.0in}The second lemma bounds the G-norm of vectors.
\begin{lemma}\label{glt1}
For any $\alpha \in \R^n$,  $\|\alpha\|_G \leq \|\alpha\|_1 \leq \sqrt{n}\|\alpha\|_2$.
\end{lemma}\vspace{-0.2in}
\begin{proof} Using the triangle inequality of norms, we get
\begin{eqnarray*}
\sqrt{\alpha^\top G \alpha} &=& \sqrt{\left \langle \langle \alpha, Y\tphi_X \rangle, \langle \alpha, Y\tphi_X \rangle \right \rangle_K}\\
&=&  \| \sum_i \alpha_i y_i \tphi_{x_i}\|_K \leq \sum_i \|\alpha_i y_i \tphi_{x_i}\|_K\\
& \leq& \sum_i |\alpha_i| \left\|y_i \frac{\phi_{x_i}}{\sqrt{K(x_i,x_i)}}\right\|_K = \sum_i |\alpha_i|,
\end{eqnarray*}
where we used $\langle \phi_{x_i},\phi_{x_i}\rangle_K = K(x_i,x_i)$.
\end{proof}

\vspace{0.0in}The third lemma gives a new perspective on the margin.
\begin{lemma}\label{marginlem}When $\rho_K>0$, $f$ maximizes the margin iff $\rho_K f$ optimizes $L(f)$. Hence, the margin is equivalently
$$\rho_K = \sup_{\|\alpha\|_G = 1} \inf_{p \in \Delta_n} \langle \alpha,p \rag \leq \|p\|_G \mbox{\ \ \ \ for all } p \in \Delta_n.$$
\end{lemma}\vspace{-0.2in}
\begin{proof} 
%The argument here is subtle - one cannot immediately substitute $f = \langle \alpha, Y\tphi_X \rangle$ into Eq.~(\ref{margin}) because such a decomposition only holds at the optimum of Eq. (\ref{loss}), but the margin could involve a $\sup$ over many other functions. 
Let $f_\rho$ be any function with $\|f_\rho\|_K=1$ that achieves the max-margin $\rho_K>0$. Then, it is easy to plug $\rho_K f_\rho$ into Eq. (\ref{loss}) and verify that $L(\rho_K f_\rho) = -\half \rho_K^2$ and hence $\rho_K f_\rho$ minimizes $L(f)$. 

Similarly, let $f_L$ be any function that minimizes $L(f)$, i.e. achieves the value $L(f_L) = -\half \rho_K^2$. Defining $t:= \|f_L\|_K$, and examining Eq.~(\ref{minloss}), we see that $L(f_L)$ cannot achieve the value $-\half \rho_K^2$ unless $t = \rho_K$ and $\sup_{p \in \Delta_n} \left \langle -Y\tf_L(X), p \right \rangle = -\rho_K^2$ which means that $f_L/\rho_K$ must achieve the max-margin.

%Hence $f$ maximizes the margin iff $\rho_K f$ optimizes $L(f)$, so 
Hence considering only $f = \sum_i \alpha_i y_i \tphi_{x_i}$ is acceptable for both. Plugging this into Eq.~(\ref{margin}) gives the equality and %the inequality is a result of Cauchy-Schwartz.
\begin{eqnarray*}
\rho_K &=& \inf_{p \in \Delta_n} \sup_{\|\alpha\|_G = 1} \langle \alpha,p \rag \leq \sup_{\|\alpha\|_G=1} \langle \alpha ,p \rag \\
&\leq& \|p\|_G \mbox{ \ \ by applying Cauchy-Schwartz }
%&=& \Big \langle \langle \alpha, Y\tphi_X \rangle, \langle p, Y\tphi_X \rangle \Big \rangle_K \leq
\end{eqnarray*}
(can also be seen by going back to function space).
\end{proof}
%Substituting $f = \langle \alpha, Y\tphi_X \rangle$ into Eq.~\ref{margin} yields the first equation. One can immediately infer the second, but it is useful to see the following direct proof
%\begin{eqnarray*}
%\rho_K &=&  \sup_{\|f\|_K = 1}  \inf_{p \in \Delta_n} \left \langle  \langle f, Y\tphi_X \rangle_K, p \right \rangle\\
%&=& \inf_{p \in \Delta_n}  \sup_{\|f\|_K = 1} \left \langle f, \langle Y\tphi_X,p \rangle  \right \rangle_K\\
%&\leq&  \sup_{\|f\|_K = 1} \left \langle f, \langle Y\tphi_X,p \rangle  \right \rangle_K \leq \|\langle Y\tphi_X,p \rangle\|_K \\
%&=&\sqrt {\left \langle \sum_i p_i y_i \tphi_{x_i}, \sum_j p_j y_j\tphi_{x_j} \right \rangle_K}\\
%&=& \sqrt{\sum_i \sum_j p_i p_j \langle y_i\tphi_{x_i}, y_j\tphi_{x_j} \rangle_K } = \sqrt{p^\top G p},
%\end{eqnarray*}
%where we used Cauchy-Schwarz inequality.

\section{Smoothed Normalized Kernel Perceptron}
Define the distribution over the worst-classified points
\begin{eqnarray}
p(f) &:=& \arg\min_{p \in \Delta_n} \left \langle Y\tf(X), p \right \rangle  \nonumber\\
\mbox{or \ \ \ } p(\alpha) &:=& \arg\min_{p \in \Delta_n} \langle \alpha, p \rag. \label{pa}
\end{eqnarray}
\vspace{-0.2in}
\begin{algorithm}[]
   \caption{Normalized Kernel Perceptron (NKP)}
\begin{algorithmic}
   \STATE Set $\alpha_0 := 0$
   \FOR{$k=0,1,2,3,...$}
   \IF{$G\alpha_k > \zeros_n$} 
   \STATE Exit, with $\alpha_k$ as solution
   \ELSE
   \STATE $\theta_k := \frac{1}{k+1}$
  % \STATE $\left[ f_{k+1} := (1-\theta_k) f_k + \theta_k \langle y\tphi_X, p_{f_k} \rangle \right]$
   \STATE $\alpha_{k+1} := (1-\theta_k) \alpha_k + \theta_k p(\alpha_k)$
   \ENDIF
   \ENDFOR
\end{algorithmic}
\end{algorithm}\vspace{-0.2in}
%Implicitly, one is starting at $f_0 = 0$, and performing
\begin{eqnarray*}
\mbox{Implicitly \ \ \ }f_{k+1} &=& (1-\theta_k)f_k + \theta_k \langle Y \tphi_X, p({f_k}) \rangle \\
&=& f_k - \theta_k \left( f_k - \langle Y\tphi_X, p({f_k}) \rangle \right)\\
&=& f_k - \theta_k \partial L(f_k)
\end{eqnarray*}
and hence the Normalized Kernel Perceptron (NKP) is a \textit{subgradient algorithm} to minimize $L(f)$ from Eq. (\ref{loss}). %It retains the properties of the classical perceptron algorithm since its function iterates are $\tfrac1{k+1}$ times the latter's. 

\textbf{Remark.} Lemma \ref{marginlem} yields deep insights. Since NKP can get arbitrarily close to the minimizer of strongly convex $L(f)$, it also gets arbitrarily close to a margin maximizer. It is known that it finds a perfect classifier in $1/\rho_K^2$ iterations - we now additionally infer that it will continue to improve to find an approximate max-margin classifier. While both classical and normalized Perceptrons find perfect classifiers in the same time, the latter is guaranteed to improve.%When NKP is seen through the light of minimizing $L(f)$ (it is not usually seen this way), it is also clear why it is enough for it to look at the span of the data points.

\textbf{Remark.} $\alpha_{k+1}$ is always a probability distribution. Curiously, a guarantee that the solution will lie in $\Delta_n$ is \textit{not} made by the Representer Theorem in Eq. (\ref{alpha}) - any $\alpha \in \R^n$ could satisfy Lemma \ref{Ga}. However, since NKP is a subgradient method for minimizing Eq.~(\ref{loss}), we know that we will approach the optimum while only choosing $\alpha \in \Delta_n$.

Define the smooth minimizer analogous to Eq.~(\ref{pa}) as
\begin{eqnarray}\label{pmuf}
p_\mu(\alpha) &:=& \arg\min_{p \in \Delta_n} \Big\{ \langle \alpha, p \rag + \mu d(p) \Big\}\\
&=& \frac{e^{-G\alpha/\mu}}{\|e^{-G\alpha/\mu}\|_1},\nonumber\\
\mbox{where \ \  } d(p) &:=& \sum_i p_i \log p_i + \log n
\end{eqnarray}
is $1$-strongly convex with respect to the $\ell_1$-norm \citep{N05}.
\begin{algorithm}[]
   \caption{Smoothed Normalized Kernel Perceptron}
\begin{algorithmic}
  \STATE Set $\alpha_0 = \ones_n/n, \ \mu_0 := 2, \ p_0 := p_{\mu_0}({\alpha_0})$
   \FOR{$k=0,1,2,3,...$}
 %  \IF{$  \langle f, y \tphi_X \rangle_K > 0$} 
   \IF{$G\alpha_k > \zeros_n $}
   \STATE Halt: $\alpha_k$ is solution to Eq.~(\ref{alpha})
%   \ELSIF {$\|p_k\|_G < \epsilon$}
%   \STATE Return $p_k$
   \ELSE
   \STATE $\theta_k := \frac{2}{k+3}$
%   \STATE $\left[f_{k+1} := \bar{\theta}_k \bigg( f_k + \theta_k \langle y\tphi_X,p_k\rangle \bigg) + \theta_k^2 \langle y\tphi_X, p^{\mu_k}_{f_k,q} \rangle \right]$
   \STATE $\alpha_{k+1} := (1-\theta_k)(\alpha_k + \theta_kp_k) + \theta_k^2p_{\mu_k}({\alpha_k}) $
   \STATE $\mu_{k+1} = (1-\theta_k) \mu_k$
   \STATE $p_{k+1} := (1-\theta_k) p_k + \theta_k p_{\mu_{k+1}}(\alpha_{k+1})$
   \ENDIF
   \ENDFOR
\end{algorithmic}
\end{algorithm}
Define a smoothened loss function as in Eq.~(\ref{lossag})
$$
L_{\mu} (\alpha) = \sup_{p \in \Delta_n} \bigg\{ - \langle \alpha, p \rag - \mu d(p) \bigg\} + \tfrac1{2} \|\alpha\|_G^2.
$$
Note that the maximizer above is precisely $p_\mu(\alpha)$.

\begin{lemma}[Lower Bound]\label{LB}
At any step $k$, we have $$L_{\mu_k} (\alpha_k) \geq L (\alpha_k) - \mu_k \log n.$$
\end{lemma}
\begin{proof}
First note that $\sup_{p \in \Delta_n} d(p) = \log n$. Also,
\begin{eqnarray*}
&&\sup_{p \in \Delta_n} \big\{ - \langle \alpha, p \rag - \mu d(p) \big\}\\
&\geq& \sup_{p \in \Delta_n} \big\{ - \langle \alpha, p \rag  \big\} - \sup_{p \in \Delta_n} \big\{\mu d(p) \big\}.
\end{eqnarray*}
Combining these two facts gives us the result.
\end{proof}\vspace{0.1in}

\begin{lemma}[Upper Bound]\label{UB}
In any round $k$, SNKP satisfies $$L_{\mu_k} (\alpha_k) \leq -\half\|p_k\|_G^2.$$
\end{lemma}
\begin{proof}
We provide a concise, self-contained and unified proof by induction in the Appendix for Lemma \ref{UB} and Lemma \ref{rhoplus}, borrowing ideas from Nesterov's excessive gap technique \citep{N05} for smooth minimization of structured non-smooth functions.
\end{proof}\vspace{0.1in}

Finally, we combine the above lemmas to get the following theorem about the performance of SNKP.

\begin{theorem}\label{snkp}
The SNKP algorithm finds a perfect classifier $f \in \F_K$ when one exists in $O\left(\tfrac{\sqrt{\log n}}{\rho_K}\right)$ iterations.
\end{theorem}
\begin{proof}
Lemma \ref{LB} gives us for any round $k$,
$$L_{\mu_k} (\alpha_k) \geq L (\alpha_k) - \mu_k \log n.$$
From Lemmas \ref{marginlem}, \ref{UB} we get
\begin{equation*}
L_{\mu_k} (\alpha_k) \leq -\tfrac1{2} p_k^\top G p_k \leq -\tfrac1{2} \rho_K^2.
\end{equation*}
Combining the two equations, we get that 
$$
L(\alpha_k) \leq \mu_k \log n -\tfrac1{2} \rho_K^2.
$$
Noting that $\mu_k = \frac{4}{(k+1)(k+2)} < \frac{4}{(k+1)^2}$, we see that $L(\alpha_k) < 0$ (and hence we solve the problem by Lemma~\ref{Ga}) after at most $k = 2\sqrt{2 \log n}/\rho_K$ steps.
\end{proof}

\section{Infeasible Problems}

What happens when the points are not separable by any function $f \in \F_K$? We would like an algorithm that terminates with a solution when there is one, and terminates with a certificate of non-separability if there isn't one. The idea is based on theorems of the alternative like Farkas' Lemma, specifically a version of Gordan's theorem \citep{C83}:\vspace{0.1in}

\begin{lemma}[Gordan's Thm] Exactly one of the following two statements can be true
\begin{enumerate}
\item Either there exists a $w \in \R^d$ such that for all $i$, $$y_i(w^\top x_i) > 0,$$
\item Or, there exists a $p \in \Delta_n$ such that 
\begin{equation}\label{lindual}
\| XY p \|_2 = 0,
\end{equation}
or equivalently $\sum_i p_i y_i x_i = 0.$
\end{enumerate}
\end{lemma}
As mentioned in the introduction, the primal problem can be interpreted as finding a vector in the interior of the dual cone of $cone\{y_ix_i\}$, which is infeasible the dual cone is flat i.e. if $cone\{y_ix_i\}$ is not pointed, which happens when the origin is in the convex combination of $y_ix_i$s.

We will generalize the following algorithm for linear feasibility problems, that can be dated back to Von-Neumann, who mentioned it in a private communication with Dantzig, who later studied it himself \citep{D92}.

%\subsection{Von Neumann Algorithm}

\begin{algorithm}[]
   \caption{Normalized Von-Neumann (NVN)}
   \label{alg:vn}
\begin{algorithmic}
   \STATE Initialize $p_0 = \ones_n/n, w_0 = XYp_0$
   \FOR{$k=0,1,2,3,...$}
   \IF{$\|XYp_k\|_2 \leq \epsilon$} 
   \STATE Exit and return $p_k$ as an $\epsilon$-solution to (\ref{lindual})
%   \ELSIF {$YX^\top w_k > 0$}
%   \STATE Exit and return $w_k$ as solution to (\ref{lf})
   \ELSE
   \STATE $j:= \arg\min_i y_i x_i^\top w_k$
   \STATE $\theta_k := \arg\min_{\lambda \in [0,1]} \|(1-\lambda)w_k + \lambda y_jx_j\|_2$
   \STATE $p_{k+1} := (1-\theta_k)p_k + \theta_k e_j$
   \STATE $w_{k+1} :=  XY p_{k+1} = (1-\theta_k)w_k + \theta_k y_j x_j$
   \ENDIF
   \ENDFOR
\end{algorithmic}
\end{algorithm}
This algorithm comes with a guarantee: \textit{If the problem (\ref{lf}) is infeasible, then the above algorithm will terminate with an $\epsilon$-approximate solution to (\ref{lindual}) in $1/\epsilon^2$ iterations.}

\cite{EF00} proved an incomparable bound - Normalized Von-Neumann (NVN) can compute an $\epsilon$-solution to (\ref{lindual}) in $O\left( \frac{1}{\rho_2^2}\log\left(\tfrac1{\epsilon}\right) \right)$ and can also find a solution to the primal (using $w_k$) in $O\left( \tfrac{1}{\rho_2^2} \right)$ when it is feasible.

We derive a smoothed variant of NVN in the next section, after we prove some crucial lemmas in RKHSs.

\subsection{A Separation Theorem for RKHSs}

While finite dimensional Euclidean spaces come with strong separation guarantees that come under various names like the separating hyperplane theorem, Gordan's theorem, Farkas' lemma, etc, the story isn't always the same for infinite dimensional function spaces which can often be tricky to deal with. We will prove an appropriate version of such a theorem that will be useful in our setting.

What follows is an interesting version of the Hahn-Banach separation theorem, which looks a lot like Gordan's theorem in finite dimensional spaces. The conditions to note here are that either $G\alpha > 0$ or $\|p\|_G = 0$.\vspace{0.1in}

\begin{theorem}
Exactly one of the following has a solution:
\begin{enumerate}\vspace{-0.1in}
\item Either $\exists f \in \F_K$ such that for all $i$, $$\frac{y_i f(x_i)}{\sqrt{K(x_i,x_i)}} = \langle f, y_i \tphi_{x_i} \rangle_K > 0 \mbox{\ \ \ i.e. \ \ } G \alpha > 0,$$
\item Or $\exists p \in \Delta_n$ such that 
\begin{equation}\label{dual}\sum_i p_i y_i \tphi_{x_i} = 0 \in \F_K \mbox{\ \ \ i.e. \ \ } \|p\|_G=0. \end{equation} 
\end{enumerate}
\end{theorem}
\begin{proof}
Consider the following set
\begin{eqnarray*}
Q &=& \left\{ (f,t) = \left(\sum_i p_i y_i \tphi_{x_i}, \sum_i p_i \right) : p \in \Delta_n \right\} \\
&=& conv \left[ (y_1 \tphi_{x_1}, 1),...,(y_n \tphi_{x_n},1) \right] \\
&\subseteq& \F_K \times \R.
\end{eqnarray*}

If (2) does not hold, then it implies that $(0,1) \notin Q$. Since $Q$ is closed and convex, we can find a separating hyperplane between $Q$ and $(0,1)$, or in other words there exists $(f,t) \in \F_K \times \R$ such that
\begin{eqnarray*}
\Big \langle (f,t) , (g,s) \Big \rangle &\geq& 0 \ \ \forall (g,s) \in Q \\
\mbox{ and } \Big \langle (f,t), (0,1) \Big \rangle &<& 0. 
\end{eqnarray*}
The second condition immediately yields $t<0$. The first condition, when applied to $(g,s)=(y_i\tphi_{x_i},1) \in Q$ yields
\begin{eqnarray*}
\langle f, y_i \tphi_{x_i} \rangle_K + t &\geq& 0 \\
\Leftrightarrow \ \ \ \ \frac{y_i f(x_i)}{\sqrt{K(x_i,x_i)}} &>& 0
\end{eqnarray*}
since $t<0$, which shows that (1) holds.

It is also immediate that if (2) holds, then (1) cannot. 
\end{proof}

Note that $G$ is positive semi-definite - infeasibility requires both that it is not positive definite, and also that the witness to $p^\top Gp = 0$ must be a probability vector. Similarly, while it suffices that $G\alpha > 0$ for some $\alpha \in \R^n$, but coincidentally in our case $\alpha$ will also lie in the probability simplex.

\subsection{The infeasible margin $\rho_K$}
Note that constraining $\|f\|_K=1$ (or $\|\alpha\|_G= 1$) in Eq.~(\ref{margin}) and Lemma \ref{marginlem} allows $\rho_K$ to be negative in the infeasible case. If it was $\leq$, then $\rho_K$ would have been non-negative because $f=0$ (ie $\alpha=0$) is always allowed. 

So what is $\rho_K$ when the problem is infeasible? Let
$$
\mathrm{conv}(Y\tphi_X) := \Big\{ \sum_i p_i y_i \tphi_{x_i} | p \in \Delta_n \Big\} \subset \F_K\vspace{-0.1in}
$$
be the convex hull of the $y_i\tphi_{x_i}$s.\vspace{0.1in}

\begin{theorem}\label{margininf}
When the primal is infeasible, the margin\footnote{We thank a reviewer for pointing out that by this definition, $\rho_K$ might always be $0$ for infinite dimensional RKHSs because there are always directions perpendicular to the finite-dimensional hull - we conjecture the definition can be altered to restrict attention to the relative interior of the hull, making it non-zero.} is
\begin{eqnarray*}
|\rho_K| = \delta_{\max}:=\sup \Big \{ \delta \ \big| \ \|f\|_K \leq \delta \Rightarrow f \in \mathrm{conv}(Y\tphi_X) \Big \}
%&=&\sup \Big \{ \delta \ \big| \ \|\alpha\|_G \leq \delta \Rightarrow \langle Y \tphi_X,\alpha \rangle \in conv(Y\tphi_X) \Big \}.
\end{eqnarray*}
\end{theorem}
\begin{proof}
\textbf{(1) For inequality $\geq$.} Choose any $\delta$ such that $f \in \mathrm{conv}(Y\tphi_X)$ for any $\|f\|_K \leq \delta$. 
Given an arbitrary $f' \in \F_K$ with $\|f'\|_K=1$, put $\tilde{f}:= -\delta f'$.

By our assumption on $\delta$, we have $\tilde{f} \in \mathrm{conv}(Y\tphi_X)$ implying there exists a $\tilde{p} \in \Delta_n$ such that  $\tilde{f} = \langle Y \tphi_X,\tilde{p}\rangle$ . Also
\begin{eqnarray*}
 \Big \langle f', \langle Y\tphi_X,\tilde{p} \rangle  \Big \rangle_K &=& \langle f', \tilde{f} \rangle_K\\
%&=& \Big\langle \langle Y\tphi_X,\alpha' \rangle, \langle Y\tphi_X,\tilde{\alpha} \rangle  \Big\rangle_K\\
&=& -\delta \|f'\|_K^2 = -\delta.
\end{eqnarray*} 
Since this holds for a particular $\tilde{p}$, we can infer
$$\vspace{-0.05in}\inf_{p \in \Delta_n} \ \Big \langle f', \langle Y\tphi_X,\tilde{p} \rangle  \Big \rangle_K  \leq -\delta.$$
Since this holds for any $f'$ with $\|f'\|_G=1$, we have 
$$
\sup_{\|f\|_K = 1}\inf_{p \in \Delta_n}  \Big \langle f', \langle Y\tphi_X,\tilde{p} \rangle  \Big \rangle_K  \leq -\delta \mbox{ \ i.e. \ } |\rho_K| \geq \delta.
$$
\textbf{(2) For inequality $\leq$.} It suffices to show $\|f\|_K \leq |\rho_K| \Rightarrow f \in \mathrm{conv}(Y\tphi_X)$. We will prove the contrapositive $ f \notin \mathrm{conv}(Y\tphi_X) \Rightarrow \|f\|_K > |\rho_K|$.
%To do so, we claim that it is enough to show that if $\langle Y\tphi_X,\alpha \rangle \notin conv(Y\tphi_X)$ then $|\rho_K| < \|\alpha\|_G$ (proof of this claim is in the Appendix). 

Since $\Delta_n$ is compact and convex,  $\mathrm{conv}(Y\tphi_X)\subset~ \F_K$ is closed and convex. Therefore if $f \notin \mathrm{conv}(Y\tphi_X)$, then there exists $g \in \F_K$ with $\|g\|_K=1$ that separates $f$ and $\mathrm{conv}(Y\tphi_X)$, i.e. for all $p \in \Delta_n$,
\begin{eqnarray*}
\langle g,f \rangle_K &<& 0 \mbox{ and } \langle g, \langle Y\tphi_X, p \rangle \rangle_K \geq 0\\
\mbox{i.e. } \langle g,f \rangle_K &<& \inf_{p \in \Delta_n} \langle g, \langle Y\tphi_X, p \rangle \rangle_K\\
&\leq& \sup_{\|f\|_K = 1}\inf_{p \in \Delta_n} \langle f, \langle Y\tphi_X, p \rangle \rangle_K =  \rho_K.
\end{eqnarray*}\vspace{-0.1in}
\begin{eqnarray*}
\mbox{Since $\rho_K < 0$ \ \ \  }|\rho_K| &<& |\langle f, g \rangle_K |\\ 
&\leq& \|f\|_K\|g\|_K = \|f\|_K. \vspace{-0.2in}
\end{eqnarray*}
\end{proof}

%When restricting our search for $f$ to the span of $\alpha$

\section{Kernelized Primal-Dual Algorithms}

%\subsection{Subroutine}

The preceding theorems allow us to write a variant of the Normalized VonNeumann algorithm from the previous section that is smoothed and works for RKHSs. Define
$$W := \Big \{ p \in \Delta_n \Big | \sum_i p_iy_i\tphi_{x_i}=0 \Big \} = \Big \{ p \in \Delta_n \Big |  \|p\|_G=0 \Big\}\vspace{-0.1in}$$
as the set of witnesses to the infeasibility of the primal. The following lemma bounds the distance of any point in the simplex from the witness set by its $\|.\|_G$ norm.\vspace{0.1in}

\begin{lemma}\label{witness}For all $q \in \Delta_n$, the distance to the witness set
\begin{equation*}
\mathrm{dist}(q,W) := \min_{w \in W} \|q-w\|_2 \leq \min\left\{\sqrt 2, \frac{\sqrt2 \|q\|_G}{|\rho_K|}\right\}.
\end{equation*}
As a consequence, $\|p\|_G = 0$ iff $p \in W$.
\end{lemma}
\begin{proof}
This is trivial for $p \in W$. For arbitrary $p \in \Delta_n \backslash W$, let $\tilde{p} := -\tfrac{|\rho_K| p}{\|p\|_G}$ so that $\|\langle Y\tphi_X,\tilde{p}\rangle\|_K = \|\tilde{p}\|_G \leq |\rho_K|$. 

Hence by Theorem \ref{margininf}, there exists $\alpha \in \Delta_n$ such that 
$$\langle Y\tphi_X,\alpha\rangle = \langle Y\tphi_X,\tilde{p} \rangle.$$ 

Let $\beta = \lambda \alpha + (1-\lambda) p$ where $\lambda = \tfrac{\|p\|_G}{\|p\|_G + |\rho_K|}$. Then 
\begin{eqnarray*}
\langle Y \tphi_X,\beta \rangle &=& \frac1{\|p\|_G + |\rho_K|} \Big \langle Y \tphi_X, \|p\|_G \alpha + |\rho_K|p \Big \rangle \\
&=& \frac1{\|p\|_G + |\rho_K|} \langle Y \tphi_X, \|p\|_G \tilde{p} + |\rho_K|p \rangle \\
&=& 0,
\end{eqnarray*}
so $\beta \in W$ (by definition of what it means to be in $W$) and
$$
\|p - \beta\|_2 = \lambda\|p-\alpha\|_2 \leq \lambda \sqrt2 \leq   \min\left\{\sqrt 2, \frac{\sqrt2 \|q\|_G}{|\rho_K|}\right\}.
$$
We take $\min$ with $\sqrt 2$ because $\rho_K$ might be $0$.
\end{proof}
Hence for the primal or dual problem, points with small G-norm are revealing - either Lemma \ref{marginlem} shows that the margin $\rho_K \leq \|p\|_G$ will be small, or if it is infeasible then the above lemma shows that it is close to the witness set.

We need a small alteration to the smoothing entropy prox-function that we used earlier. %Instead of $d(p) = \sum_i p_i \log p_i + \log n$ which is strongly convex with respect to the $\ell_1$ norm
We will now use
$$
d_q(p) = \half\|p - q\|_2^2
$$
for some given $q\in \Delta_n$, which is strongly convex with respect to the $\ell_2$ norm. This allows us to define
\begin{eqnarray*}
p^q_\mu(\alpha) &=& \arg \min_{p \in \Delta_n} \left \langle G\alpha, p \right \rangle + \frac{\mu}{2} \|p-q\|_2^2,\\
L^q_{\mu} (\alpha) &=& \sup_{p \in \Delta_n} \bigg\{ - \langle \alpha, p \rangle_G - \mu d_q(p) \bigg\} + \tfrac1{2} \|\alpha\|_G^2,
\end{eqnarray*}
which can easily be found by sorting the entries of $q - \frac{G\alpha}{\mu}$.

\begin{algorithm}[]
   \caption{Smoothed Normalized Kernel Perceptron-VonNeumann ($SNKPVN(q,\delta)$)}
\begin{algorithmic}
   \STATE Input $q \in \Delta_n$, accuracy $\delta > 0$
   \STATE Set $\alpha_0 = q, \ \mu_0 := 2n, \ p_0 := p_{\mu_0}^q(\alpha_0)$
   \FOR{$k=0,1,2,3,...$}
 %  \IF{$  \langle f, y \tphi_X \rangle_K > 0$} 
   \IF{$G\alpha_k > \zeros_n $}
   \STATE Halt: $\alpha_k$ is solution to Eq.~(\ref{alpha})
   \ELSIF {$\|p_k\|_G < \delta$}
   \STATE Return $p_k$
   \ELSE
   \STATE $\theta_k := \frac{2}{k+3}$
%   \STATE $\left[f_{k+1} := \bar{\theta}_k \bigg( f_k + \theta_k \langle y\tphi_X,p_k\rangle \bigg) + \theta_k^2 \langle y\tphi_X, p^{\mu_k}_{f_k,q} \rangle \right]$
   \STATE $\alpha_{k+1} := (1-\theta_k)(\alpha_k + \theta_kp_k) + \theta_k^2 \ p_{\mu_k}^q(\alpha_k) $
   \STATE $\mu_{k+1} = (1-\theta_k) \mu_k$
   \STATE $p_{k+1} := (1-\theta_k) p_k + \theta_k \ p_{\mu_{k+1}}^q(\alpha_{k+1})$
   \ENDIF
   \ENDFOR
\end{algorithmic}
\end{algorithm}

%The following two lemmas capture what SNKPVN does in the two cases.

When the primal is feasible, SNKPVN is similar to SNKP.
\begin{lemma}[When $\rho_K>0$ and $\delta < \rho_K$]\label{rhoplus}
 For any $q~ \in \Delta_n$,
$$-\half \|p_k\|_G^2 ~\geq~ L^q_{\mu_k}(\alpha_k) ~\geq~ L(\alpha_k)-~\mu_k.$$ Hence SNKPVN finds a separator $f$ in $O\left(\tfrac{\sqrt{n}}{\rho_K} \right)$ iterations.
\end{lemma}\vspace{-0.1in}
\begin{proof}
We give a unified proof for the first inequality and Lemma \ref{UB} in the Appendix. The second inequality mimics Lemma \ref{LB}. The final statement mimics Theorem \ref{snkp}.
\end{proof}\vspace{0.1in}

The following lemma captures the near-infeasible case.
\begin{lemma}[When $\rho_K < 0$ or $\delta > \rho_K$]\label{rhominus}
 For any $q \in \Delta_n$,
$$-\half \|p_k\|_G^2 ~\geq~  L^q_{\mu_k}(\alpha_k)~ \geq~ -\half \mu_k \mathrm{dist}(q,W)^2.$$
Hence SNKPVN finds a $\delta$-solution in at most $O\left( \min\left\{\frac{\sqrt n}{\delta},\frac{\sqrt{n}\|q\|_G}{\delta|\rho_K|} \right\} \right)$ iterations.
\end{lemma}\vspace{-0.1in}
\begin{proof}
The first inequality is the same as in the above Lemma \ref{rhoplus}, and is proved in the Appendix.
\begin{eqnarray*}
L^q_{\mu_k}(\alpha_k) &=& \sup_{p \in \Delta_n} \bigg\{ - \langle \alpha, p \rangle_G - \mu_k d_q(p) \bigg\} + \tfrac1{2} \|\alpha\|_G^2\\
&\geq&  \sup_{p \in W} \bigg\{ - \langle \alpha, p \rangle_G - \mu_k d_q(p) \bigg\}\\
&=& \sup_{p \in W} \bigg\{ - \half\mu_k \|p-q\|_2^2 \bigg\}\\
&=&-\half\mu_k \mathrm{dist}(q,W)^2\\
&\geq& -\mu_k \min\left\{2,\tfrac{\|q\|_G^2}{|\rho_K|^2}\right\} \mbox{\ \ \ \ \ \ using Lemma \ref{witness}}.
\end{eqnarray*}
Since $\mu_k = \frac{4n}{(k+1)(k+2)} \leq \frac{4n}{(k+1)^2}$ we get 
$$\|p_k\|_G \leq \frac{2\sqrt{n}}{(k+1)}\min\left\{\sqrt 2,\frac{\|q\|_G}{\rho_K}\right\}.$$
Hence $\|p\|_G \leq \delta$ after $\frac{2\sqrt{n}}{\delta} \min\left\{\sqrt 2,\frac{\|q\|_G}{\rho_K}\right\}$ steps.
\end{proof}

%\subsection{Iterated Smoothed Normalized Kernel Perceptron-VonNeumann Algorithm}
Using SNKPVN as a subroutine gives our final algorithm.
\begin{algorithm}[]
   \caption{Iterated Smoothed Normalized Kernel Perceptron-VonNeumann ($ISNKPVN(\gamma,\epsilon)$)}
\begin{algorithmic}
   \STATE Input constant $\gamma > 1$, accuracy $\epsilon > 0$
   \STATE Set $q_0 := \ones_n/n$
   \FOR{$t=0,1,2,3,...$}
   \STATE {$\delta_t := \|q_t\|_G/\gamma$}
   \STATE {$q_{t+1} := SNKPVN (q_t,\delta_t)$}
   \IF {$\delta_t < \epsilon$}
   \STATE Halt; $q_{t+1}$ is a solution to Eq. (\ref{dual})
   \ENDIF
   \ENDFOR
\end{algorithmic}
\end{algorithm}

\begin{theorem} Algorithm ISNKPVN satisfies
\begin{enumerate} 
\item If the primal (\ref{prim2}) is feasible and  $\epsilon < \rho_K$, then each call to SNKPVN halts in at most $\frac{2\sqrt{2n}}{\rho_K}$ iterations. Algorithm ISNKPVN finds a solution in at most $\frac{\log (1/\rho_K)}{\log (\gamma)}$ outer loops, bounding the total iterations by
$$
O\left( \frac{\sqrt n}{\rho_K} \log \left(\frac1{\rho_K}\right) \right).
$$
\item If the dual (\ref{dual}) is feasible or $\epsilon > \rho_K$, then each call to SNKPVN halts in at most $O\left(\min\left\{\frac{\sqrt n}{\epsilon},\frac{\sqrt n}{|\rho_K|}\right\}\right)$ steps. Algorithm ISNKPVN finds an $\epsilon$-solution in at most $\frac{\log(1/\epsilon)}{\log (\gamma)}$ outer loops, bounding the total iterations by
$$
O\left( \min\left\{\frac{\sqrt n}{\epsilon}, \frac{\sqrt n}{|\rho_K|} \right\} \log \left(\frac1{\epsilon}\right) \right).
$$
\end{enumerate}
\end{theorem}
\begin{proof}
First note that if ISNKPVN has not halted, then we know that after $t$ outer iterations, $q_{t+1}$ has small G-norm:
\begin{equation}\label{delta}
\|q_{t+1}\|_G \leq \delta_{t} \leq \frac{\|q_0\|_G}{\gamma^{t+1}}.
\end{equation}
The first inequality holds because of the inner loop return condition, the second because of the update for $\delta_t$.
\begin{enumerate}
\item Lemma \ref{marginlem} shows that for all $p$ we have $\rho_K \leq \|p\|_G$, so the inner loop will halt with a solution to the primal as soon as $\delta_t \leq \rho_K$ (so that $\|p\|_G < \delta_t \leq \rho_K$ cannot be satisfied for the inner loop to return). From Eq. (\ref{delta}), this will definitely happen when $\frac{\|q_0\|_G}{\gamma^{t+1}} \leq \rho_K$, ie within $T = \frac{\log(\|q_o\|_G/\rho_K)}{\log(\gamma)}$ iterations. By Lemma \ref{rhoplus}, each iteration runs for at most $\frac{2\sqrt{2n}}{\rho_K}$ steps.
\item We halt with an $\epsilon$-solution when $\delta_t < \epsilon$, which definitely happens when $\frac{\|q_0\|_G}{\gamma^{t+1}} < \epsilon$, ie within  $T ~=~ \frac{\log(\|q_o\|_G/\epsilon)}{\log(\gamma)}$ iterations. Since $\frac{\|q_t\|_G}{\delta_t} = \gamma$, by Lemma \ref{rhominus}, each iteration runs for at most $O\left(\min\left\{\frac{\sqrt n}{\epsilon}, \frac{\sqrt n}{|\rho_K|} \right\}\right)$ steps.
\end{enumerate}
\vspace{-0.3in}
\end{proof}

\section{Discussion}

The SNK-Perceptron algorithm presented in this paper has a convergence rate of $\frac{\sqrt{\log n}}{\rho_K}$ and the Iterated SNK-Perceptron-Von-Neumann algorithm has a $\min\left\{\frac{\sqrt n}{\epsilon},\frac{\sqrt n}{|\rho_K|}\right\}$ dependence on the number of points. Note that both of these are independent of the underlying dimensionality of the problem. We conjecture that it is possible to reduce this dependence to $\sqrt{\log n}$ for the primal-dual algorithm also, without paying a price in terms of the dependence on margin $1/\rho$ (or the dependence on $\epsilon$).

It is possible that tighter dependence on $n$ is possible if we try other smoothing functions instead of the $\ell_2$ norm used in the last section. Specifically, it might be tempting to smooth with the $\|.\|_G$ semi-norm and define:
$$p_\mu^q(\alpha) = \arg \min_{p \in \Delta_n} \langle \alpha, p\rangle_G + \frac{\mu}{2}\|p-q\|_G^2$$
One can actually see that the proofs in the Appendix go through with no dimension dependence on $n$ at all! However, it is not possible to solve this in closed form - taking $\alpha = q$ and $\mu=1$ reduces the problem to asking $$p^q(q) = \arg \min_{p \in \Delta_n} \half \|p\|_G^2$$
which is an oracle for our problem as seen by equation (\ref{dual}) - the solution's G-norm is $0$ iff the problem is infeasible.

In the bigger picture, there are several interesting open questions. The ellipsoid algorithm for solving linear feasibility problems has a logarithmic dependence on $1/\epsilon$, and a polynomial dependence on dimension. Recent algorithms involving repeated rescaling of the space like the one by \cite{DV08} have logarithmic dependence on $1/\rho$ and polynomial in dimension. While both these algorithms are  poly-time under the real number model of computation of \cite{BCSS98}, it is unknown whether there is any algorithm that can achieve a polylogarithmic dependence on the margin/accuracy, and a polylogarithmic dependence on dimension. This is strongly related to the open question of whether it is possible to learn a decision list polynomially in its binary description length.

One can nevertheless ask whether rescaled \textit{smoothed} perceptron methods like \cite{DV08} can be lifted to RKHSs, and whether using an iterated smoothed kernel perceptron would yield faster rates. The recent work \cite{SP13new} is a challenge to generalize - the proofs relying on geometry involve arguing about volumes of balls of functions in an RKHS - we conjecture that it is possible to do, but we leave it for a later work.

\subsection*{Acknowledgements}
We thank 
Negar Soheili and Avrim Blum for insightful discussions.

\bibliography{MKSP}
\bibliographystyle{agsm}

\appendix

\section{Unified Proof By Induction of Lemma 5, 8: $L_{\mu_{k}}(\alpha_k) \leq -\half \|p_k\|_G^2$}
Let $d(p)$ be $1$-strongly convex with respect to the $\#$-norm, ie $d(q) - d(p) - \langle \nabla d(p), q-p \rangle \geq \half\|q-p\|_\#^2$for any $p,q \in \Delta_n$. Let the $\#$-norm be lower bounded by the G-norm as $\|p\|^2_G \leq \lambda_\# \|p\|^2_\#$. For $d(p) = \sum_i p_i \log p_i + \log n$, $\#$ is the $1$-norm, $\lambda_\#=1$ and $p^* = \onebyn$. For $d(p) = \half \|q-p\|_2^2$, $\#$ is the $2$-norm, $\lambda_\#=n$ and $p^*=q$. Choose $\mu_0 = 2\lambda_\#$. 

Let the smoothed minimizer be defined by $p_\mu(\alpha) := \arg\min_{p \in \Delta_n} \langle G\alpha,p\rangle + \mu d(p)$, and $p^* := \arg\min_{p \in \Delta_n} d(p)$. The optimality condition of $p_\mu(\alpha)$ and $p^*$ (the gradient is perpendicular to any feasible direction) is that for any $r \in \Delta_n$, 
\begin{eqnarray}
\langle G\alpha + \mu \nabla d(p_\mu(\alpha)), r - p \rangle &=& 0\\
\langle \nabla d(p^*), r - p \rangle &=& 0 \Rightarrow d(p_0) \geq \half \|p_0-p^*\|_\#^2.
\end{eqnarray}\vspace{-0.2in}
\begin{eqnarray*}
\mbox{For } k=0: \ \ \ \ \ -\half \|p_0\|_G^2 %&=& -\half\|(p_0 - p^*) + p^*\|_G^2 \\
&=& -\half\|p_0 - p^*\|_G^2  - \langle p^*, p_0 - p^* \rangle_G - \half \|p^*\|_G^2 \mbox{\ \ \ \ \  writing $p_0 = (p_0 - p^*) + p^*$}\\
&\geq& -\tfrac{\lambda_\#}{2} \|p_0 - p^*\|_\#^2 -\langle p^*, p_0 \rangle_G  + \half \|p^*\|_G^2 \mbox{\ \ \ \ \ \ \ \ \ \ \ using $\|p\|_G^2 \leq \lambda_\# \|p\|_\#^2$} \\
&\geq& -\mu_0 d(p_0) - \langle \alpha_0,p_0 \rangle_G + \half \|\alpha_0\|_G^2 \mbox{\ \ \ \ \ \ \ \ \ \ \ \ \ \ \ \ \ \ \ \  adding $-\tfrac{\lambda_\#}{2} \|p_0 - p^*\|_1^2$, using Eq. (2)}\\
&=& L_{\mu_0}(\alpha_0).
\end{eqnarray*}
Assume it holds upto $k$. We drop index $k$, and write $x_{+}$ for $x_{k+1}$.  Let $\phat = (1-\theta)p + \theta p_\mu(\alpha)$ so $\ap = (1-\theta)\alpha + \theta \phat$. \ \ (3)
\begin{eqnarray*}
L_{\mu_+}(\ap) &=& \half \|\ap\|_G^2 - \Big \la \ap, \pump(\ap) \Big \rag - \mu_+ d(\pump(\ap))\\
&=& \half \big\|(1-\theta)\alpha + \theta \phat \big\|_G^2 - \theta \Big\la \phat, \pump(\ap) \Big\rag  - (1-\theta)\bigg[ \Big\la \alpha, \pump(\ap) \Big \rag + \mu d(\pump(\ap))\bigg] \ \ \ \ \mbox{using Eq. (3)}\\
&\leq& (1-\theta) \bigg[ \half \|\alpha\|_G^2 - \Big\la \alpha, \pump(\ap) \Big\rag - \mu d(\pump(\ap)) \bigg]_1 + \theta \bigg[ -\half \|\phat\|_G^2 - \Big\la \phat, \pump(\ap) -\phat \Big\rag \bigg],
\end{eqnarray*}
where we used the convexity of $\|.\|_G^2$. Recall $p_+ = (1-\theta)p + \theta \pump(\ap)$, so that $p_+ - \phat= \theta(\pump(\ap) - p_\mu(\alpha))$. \ \ \ \ \ (4)
\begin{eqnarray*}
\Big [. \Big ]_1 &=& \Big [ \half  \|\alpha\|_G^2 - \Big \la \alpha, p_\mu(\alpha) \Big \rag - \mu d(p_\mu(\alpha)) \Big] - \Big \la \alpha, \pump(\ap) - p_\mu(\alpha) \Big \rag - \mu \Big[ d(\pump(\ap)) - d(p_\mu(\alpha)) \Big]\\
&=& L_\mu(\alpha) - \mu \Big[ d(\pump(\ap)) - d(p_\mu(\alpha)) - \Big \langle \nabla d(p_\mu(\alpha)), \pump(\ap) - p_\mu(\alpha) \Big \rangle \Big] \mbox{ \ \ \ using Eq. (1)}\\
&\leq& -\half \|p\|_G^2 - \tfrac{\mu}{2} \|\pump(\ap) - p_\mu(\alpha)\|_\#^2 \mbox{\ \ \ \ \ \ \ \ \ \ \ \ \ \ \ \ \ \ \ \  \ \ \ using strong convexity of $d(p)$}\\
&\leq& -\half \|\phat + (p - \phat)\|_G^2 - \tfrac{\mu}{2\lambda_\#} \|\pump(\ap) - p_\mu(\alpha)\|_G^2 \mbox{\ \ \ \ \ \ \ \ \ \ \ using $\|p\|_G^2 \leq \lambda_\# \|p\|_\#^2$}\\
&\leq&  -\half \|\phat\|_G^2 - \Big\langle \phat,p-\phat \Big\rag - \frac{\mu}{2\lambda_\# \theta^2} \|p_+ - \phat\|_G^2 \mbox{\ \ \ \ \ \ \ \ \ \ \ \ \ \ using Eq. (4) and dropping a $-\half \|p-\phat\|_G^2$ term.}
\end{eqnarray*}
Using $(1-\theta) (p - \phat) = -\theta(p_{\mu}(\alpha) - \phat)$ and substituting back, 
\begin{eqnarray*}
L_{\mu_+}(\ap) &\leq& (1-\theta) \bigg[  -\half \|\phat\|_G^2 + \tfrac{\theta}{1-\theta}\Big\langle \phat,p_\mu(\alpha)-\phat \Big\rag - \frac{\mu}{2\lambda_\#\theta^2} \|p_+ - \phat\|_G^2 \bigg] + \theta \bigg[ -\half \|\phat\|_G^2 - \Big\la \phat, \pump(\ap) -\phat \Big\rag \bigg]\\
&=& -\half \|\phat\|_G^2 - \theta \Big\la \phat, \pump(\ap) -p_\mu(\alpha) \Big\rag - \frac{\mu (1-\theta)}{2\lambda_\#\theta^2} \|p_+ - \phat\|_G^2\\
&\leq& -\half \|\phat\|_G^2 -  \Big\la \phat,  p_+ - \phat \Big\rag - \half\|p_+ - \phat\|_G^2 \mbox{\ \ \ \ \ \ \ \ using Eq. (4) and $\tfrac{\theta^2}{1-\theta} = \frac{4}{(k+1)(k+3)} \leq  \frac{4}{(k+1)(k+2)} = \frac{\mu}{\lambda_\#}$}\\
&=& -\half \|p_+\|_G^2.
\end{eqnarray*}
This wraps up our unified proof for both settings.

\end{document}